\documentclass{article}

\usepackage{times}
\usepackage{graphicx} 
\usepackage{subfigure} 

\usepackage{natbib}

\usepackage{algorithm}
\usepackage{algorithmic}

\renewcommand{\algorithmicrequire}{\textbf{Initialize:}}   
\renewcommand{\algorithmicensure}{\textbf{Iteration:}} 
\renewcommand{\algorithmicinput}{\textbf{Input:}} 
\renewcommand{\algorithmicoutput}{\textbf{Output:}}

\usepackage{hyperref}

\usepackage{helvet}
\usepackage{courier}

\usepackage{makecell}
\usepackage{graphicx}
\usepackage{subfigure}
\usepackage{color}
\usepackage{amsfonts}
\usepackage{amsmath}
\usepackage{amssymb}

\usepackage{multirow}
\usepackage{booktabs}

\DeclareMathOperator*{\argmin}{arg\,min}

\def\bgamma{\mbox{{\boldmath $\gamma$}}}

\def\mB{{\mathcal B}}
\def\mC{{\mathcal C}}
\def\mD{{\mathcal D}}

\def\mF{{\mathcal F}}
\def\mG{{\mathcal G}}
\def\mH{{\mathcal H}}

\def\mM{{\mathcal M}}
\def\mN{{\mathcal N}}

\def\mU{{\mathcal U}}
\def\mV{{\mathcal V}}
\def\mW{{\mathcal W}}
\def\mX{{\mathcal X}}

\DeclareMathAlphabet\mathbfcal{OMS}{cmsy}{b}{n}

\def\0{{\bf 0}}
\def\1{{\bf 1}}

\def\bI{{\bf I}}

\def\bV{{\bf V}}

\def\bh{{\bf h}}

\def\br{{\bf r}}

\def\bv{{\bf v}}

\def\bx{{\bf x}}

\def\bz{{\bf z}}

\def\mmE{{\mathbb E}}

\def\mmR{{\mathbb R}}

\def\trsp{{\sf T}}

\def\st{{\mathrm{s.t.}}}

\def\bx{{\bf x}}

\def\bh{{\bf h}}

\def\bz{{\bf z}}

\def\st{{\mathrm{s.t.}}}

\usepackage{ntheorem}

\newtheorem{deftn}{Definition}
\newtheorem{thm}{Theorem}
\newtheorem*{*thm}{Theorem}

\newtheorem{lemma}{Lemma}
\newtheorem*{*lemma}{Lemma}

\newenvironment*{proof}{\textbf{Proof}\quad}{\hfill $\square$\par}

\usepackage[accepted]{icml2018}

\begin{document}

\twocolumn[
\icmltitle{Adversarial Learning with Local Coordinate Coding}



\icmlsetsymbol{equal}{*}

\begin{icmlauthorlist}
\icmlauthor{Jiezhang Cao}{equal,scut}
\icmlauthor{Yong Guo}{equal,scut}
\icmlauthor{Qingyao Wu}{equal,scut}
\icmlauthor{Chunhua Shen}{adelaide}
\icmlauthor{Junzhou Huang}{uta}
\icmlauthor{Mingkui Tan}{scut}
\end{icmlauthorlist}

\icmlaffiliation{scut}{School of Software Eigineering, South China University of Technology, China}
\icmlaffiliation{adelaide}{School of Computer Science, The University of Adelaide, Australia}
\icmlaffiliation{uta}{Tencent AI Lab, China; University of Texas at Arlington, America}
\icmlcorrespondingauthor{Mingkui Tan}{mingkuitan@scut.edu.cn}

\icmlkeywords{Adversarial Learning, Latent Manifold, Local Coordinate Coding, Generative Adversarial Networks, LCC Sampling, Generalization Bound, Intrinsic Dimension} 

\vskip 0.3in
]



\printAffiliationsAndNotice{\icmlEqualContribution} 

\begin{abstract}
	Generative adversarial networks (GANs) aim to generate realistic data from some prior distribution (\textit{e.g.}, Gaussian noises).
	However, such prior distribution is often independent of real data and thus may lose semantic information (\textit{e.g.}, geometric structure or content in images) of data. 
	In practice, the semantic information might be represented by some latent distribution learned from data, which, however, is hard to be used for sampling in GANs. 
	In this paper, rather than sampling from the pre-defined prior distribution, we propose a Local Coordinate Coding (LCC) based sampling method to improve GANs.
    We derive a generalization bound for LCC based GANs and prove that a small dimensional input is sufficient to achieve good generalization performance.
    Extensive experiments on various real-world datasets demonstrate the effectiveness of the proposed method.
\end{abstract}

\section{Introduction}
Generative Adversarial Networks (GANs) \cite{goodfellow2014gans} have been successfully applied to many tasks, such as video prediction \cite{ranzato2014video, mathieu2015deep}, image translation \cite{isola2017image, kim2017learning}, etc.
Specifically, GANs learn to generate data by playing a two-player game: a generator tries to produce samples from a simple latent distribution, and a discriminator distinguishes between the generated data and real data.

Recently, many attempts have been made to improve GANs \cite{radford2015unsupervised, arjovsky2017wasserstein, karras2017progressive}.
However, existing studies suffer from two limitations. First, many studies employ some simple prior distribution, such as Gaussian distributions \cite{goodfellow2014gans} and uniform distributions~\cite{radford2015unsupervised}.
However, such pre-defined prior distributions are often independent of the data distributions and these methods may produce images with distorted structures without sufficient semantic information.
Although such semantic information can be represented by some latent distribution,
\textit{e.g.}, extracting embeddings using an AutoEncoder~\cite{hinton2006reducing},
how to conduct sampling from this distribution still remains an open question in GANs.

Second, the generalization ability of GANs w.r.t. the dimension of the latent distribution is unknown.
In practice, we observe that the performance of GANs is sensitive to the dimension of the latent distribution.
Unfortunately, it is difficult to analyze the dimensionality of the latent distribution, since the specified prior distribution is independent of the real data.
Therefore, it is very necessary and important to explore a new method to study the dimension of latent distribution and its impacts on the generalization ability.

In this paper, relying on the manifold assumption on images~\cite{tenenbaum2000global, roweis2000nonlinear}, 
we propose a novel generative model using Local Coordinate Coding (LCC)~\cite{yu2009nonlinear} to improve GANs in generating perceptually convincing images.
First, we employ an AutoEncoder to learn embeddings lying on the latent manifold to capture the  semantic information in data.
Then, we develop a new LCC sampling method for training GANs by exploiting the local information on the latent manifold.

The contributions of this paper are summarized as follows.

First, we propose an LCC sampling method for GANs to capture the local information of data.
With the LCC sampling, the proposed scheme, called LCC-GANs, is able to sample meaningful points from the latent manifold to generate new data.

Second, we study the generalization bound of LCC-GANs based on the Rademacher complexity of the discriminator set and the error w.r.t. the intrinsic dimensionality of the manifold.
In particular, we prove that a small dimensional input is sufficient to achieve good generalization performance.
Extensive experiments on real-world datasets demonstrate the superiority of the proposed method over several state-of-the-arts.

\section{Related Studies}

Recently, Generative Adversarial Networks have shown promising performance for generating images, such as DCGANs \cite{radford2015unsupervised}, WGANs \cite{ arjovsky2017wasserstein} and Progressive GANs \cite{karras2017progressive}.
Most existing generative models seek to learn from some simple prior distribution, such as Gaussian distributions and uniform distributions, to generate samples \cite{goodfellow2014gans, arjovsky2017wasserstein, radford2015unsupervised, karras2017progressive}.
However, such prior distributions are independent of the data distributions, which may lose semantic information and lead to difficulties in analyzing the dimension of latent space.

Besides, some generative models do sampling via some learned posterior distribution. 
For example, Variational AutoEncoder (VAE) \cite{kingma2013auto}, Wasserstein AutoEncoder (WAE) \cite{tolstikhin2018wasserstein} and Adversarial AutoEncoder (AAE) \cite{makhzani2015adversarial} enforce the posterior distribution to match the prior distribution.
However, it is difficult for these methods to conduct sampling directly on the posterior distribution.
Moreover, although these methods help to make inference, overly simplified distributions would also lose semantic information.

\section{Preliminaries}

\subsection{Local Coordinate Coding}
We first introduce some definitions about local coordinate coding which will be used to develop our proposed method.
\begin{deftn} \textbf{\emph{(Lipschitz Smoothness \cite{yu2009nonlinear})}}
	A function $ f_{\theta}(\bx) $ in $ \mmR^d $ is $ (L_{\bx}, L_{f}) $-Lipschitz smooth if $ \| f(\bx') - f(\bx) \|_2 \leq L_{\bx} \| \bx - \bx' \|_2 $ and $ \| f(\bx') - f(\bx) - \nabla f(\bx)^\trsp (\bx' - \bx) \|_2 \leq L_{f} \| \bx - \bx' \|_2^2 $, where $ L_{\bx}, L_{f} > 0 $.
\end{deftn}
\begin{deftn} \textbf{\emph{(Coordinate Coding \cite{yu2009nonlinear})}} \label{definition: Coordinate Coding}
	A coordinate coding is a pair $ (\bgamma, \mC) $, where $ \mC \subset \mmR^d $ is a set of anchor points (bases), and $ \gamma $ is a map of $ \bx \in \mmR^d $ to $ \left[ \gamma_{\bv} (\bx) \right]_{\bv \in \mC} \in \mmR^{|\mC|} $ such that $ \sum_{\bv} \gamma_{\bv} (\bx) = 1 $. Then, the physical approximation of $ \bx \in \mmR^d $ is $ \br(\bx) = \sum_{\bv \in \mC} \gamma_{\bv} (\bx) \bv $.
\end{deftn}
Definition \ref{definition: Coordinate Coding} indicates that any point in $ \mmR^d $ can be represented by a linear combination of a set of anchor points.

\subsection{Latent Manifold}
High dimensional data often lie on some low dimensional manifold \cite{tenenbaum2000global, roweis2000nonlinear}.
Based on this manifold assumption, we can learn a manifold $ \mM $ embedded in the latent space $ \mmR^{d_B} $ by some manifold learning method, such as an AutoEncoder (AE)~\cite{hinton2006reducing}, to capture the semantic information of data.  
Given $ N $ training data $ \{ \bx_i \}_{i=1}^N $, we can use an Encoder to extract the embeddings $ \{ \bh_i \}_{i=1}^{N} $, where $ \bh_i = \mathrm{Encoder}(\bx_i)$.
Formally, the latent manifold can be defined as follows.

\begin{deftn} \textbf{\emph{(Latent Manifold \cite{yu2009nonlinear})}} \label{definition: manifold}
	A subset $ \mM $ embedded in the latent space $ \mmR^{d_{B}} $ is called a smooth manifold with a \textbf{intrinsic dimension} $ d := d_{\mM} $, if there exists a constant $ c_{\mM} $, such that given any $ \bh \in \mM $, there are $ d $ bases $ \bv_1(\bh), \ldots, \bv_{d}(\bh) \in \mmR^{d_B} $ so that $ \forall\; \bh' \in \mM:$
\begin{equation*}	
\begin{array}{ll}
	\inf_{\bgamma \in \mmR^{d}} \left\| \bh' - \bh - \sum_{j=1}^{d} \gamma_j \bv_j(\bh) \right\|_2 \leq c_{\mM} \| \bh' - \bh \|_2^{2}.
\end{array}
\end{equation*}
	where $ \bgamma = [\gamma_1, \ldots, \gamma_d]^{\trsp} $ is the local coding of a latent point $ \bh $ using the corresponding bases.
\end{deftn}

\begin{figure}[t]
	\centering
	\subfigure[Local function approximation based on LCC]
	{
		\label{fig:LCC_a}
		\includegraphics[width = 0.47\columnwidth]{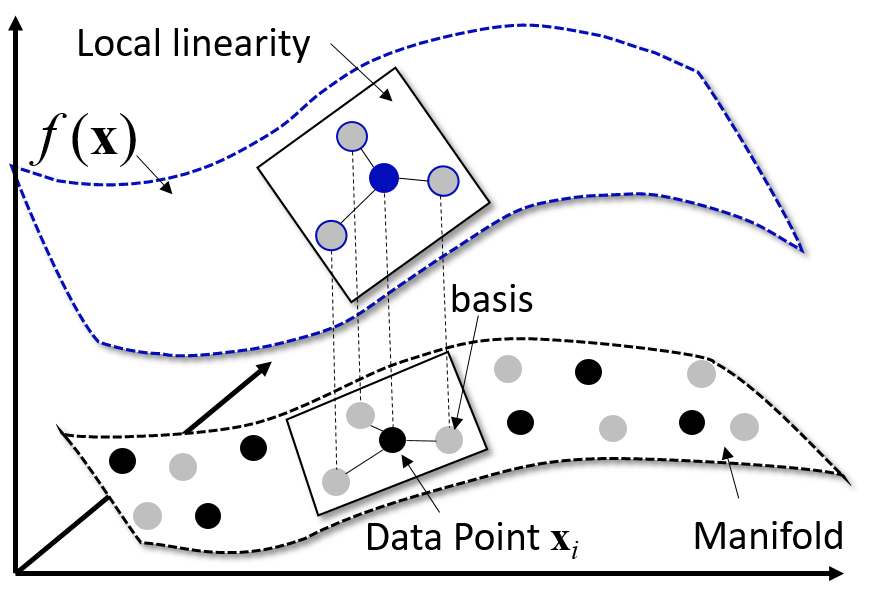}
	}
	\subfigure[Global function approximation based on LCC]{
		\label{fig:LCC_b}
		\includegraphics[width = 0.47\columnwidth]{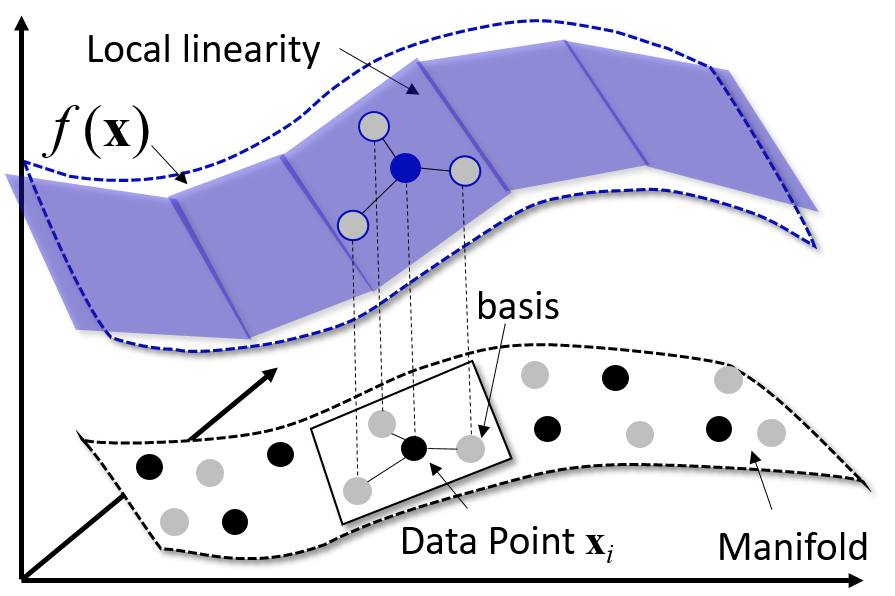}
	}
	\caption{A geometric view of Local Coordinate Coding. Given a set of local bases, if data lie on a manifold, a nonlinear function $ f(\bx) $ can be locally approximated by a linear function w.r.t. the coding. Given all bases, $ f(\bx) $ can be globally approximated. }
	\label{fig:LCC}
\end{figure}

\subsection{Generative Adversarial Networks}
We apply the neural network distance \cite{arora2017gans} to measure the similarity between two distributions.
\begin{deftn} \textbf{\emph{(Neural Network Distance \cite{arora2017gans}) }} \label{definition: F_distance}
	Let $ \mF $ be a set of neural networks
	from $ \mmR^d $ to $ [0, 1] $ and $ \phi $ be a concave measure function, then for $ D \in \mF $, the neural network distance w.r.t. $ \phi $ between two distributions $ \mu $ and $ \nu $ can be defined as
	\begin{align*}
	d_{\mF, \phi} (\mu, \nu) \small{=} \small{\sup\limits_{D \in \mF}} \left| \mathop \mmE\limits_{\bx \sim \mu} \big[ \phi(D(\bx)) \big] \small{+} \mathop \mmE\limits_{\bx \sim \nu} \big[ \phi(\widetilde{D}(\bx)) \big] \right| \small{-} \phi_c,
	\end{align*}
	where $ \phi_c = 2 \phi(\frac{1}{2}) $ is a constant with given $ \phi $ and $ \widetilde{D}(\bx) = 1 - D(\bx) $. \emph{For simplicity, we can omit the constant $ \phi_c $.}
\end{deftn}

\textbf{Objective function of general GANs. }
Given a Generator $ G_u $ and a Discriminator $ D_v $ parameterized by $ u \in \mU $ and $ v \in \mV $, where $ \mU $ and $ \mV $ are parameter spaces.
Let $ \mD_{real} $ be the real distribution of training samples $ \bx \in \mmR^d $ and $ \mD_{G_u} $ be the  distribution  generated by $ G_u $. The objective function of GANs can be defined as:
\begin{align*}
\mathop {\min}\limits_{u \in \mU} \mathop {\max}\limits_{v \in \mV} \mathop\mmE_{\bx \sim \mD_{real}} \left[ \phi (D_v(\bx)) \right] + \mathop\mmE_{\bx \sim \mD_{G_u}} \left[ \phi (1 - D_v(\bx)) \right],
\end{align*}
where $ \phi: [0, 1] \rightarrow \mmR $ is any monotone function.

\begin{figure*}[htp]
	\centering
	{
		\includegraphics[width=0.7\linewidth]{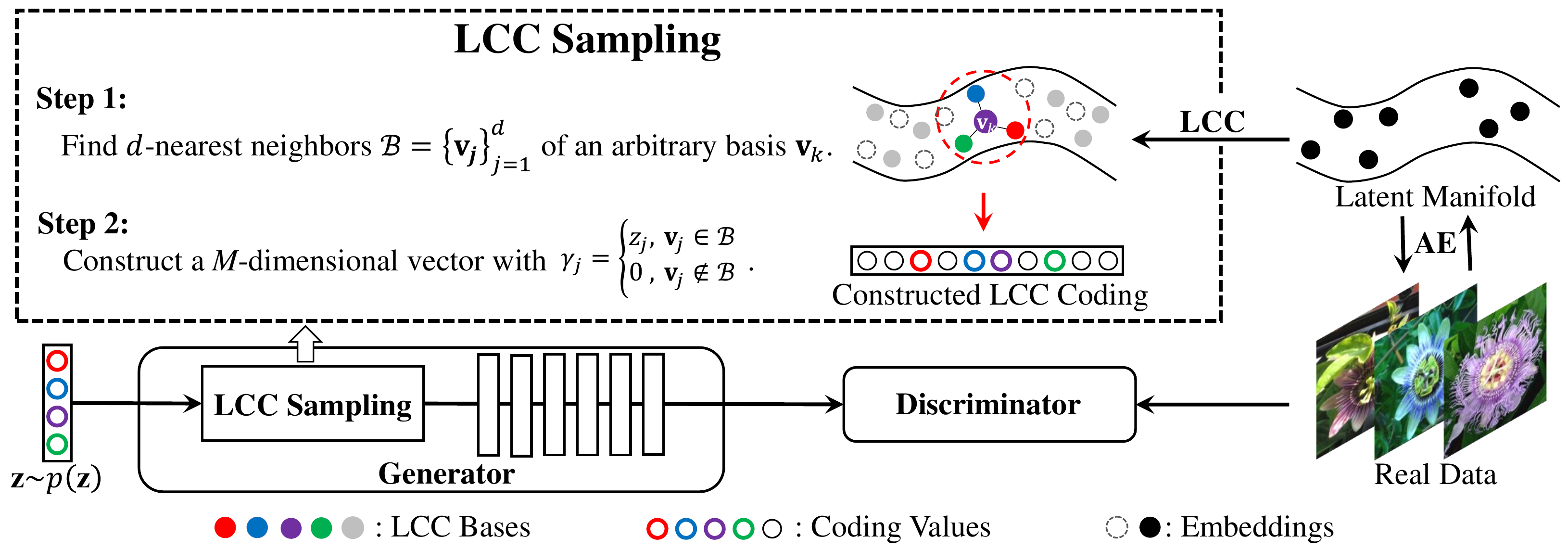}
		\caption{The scheme of the proposed LCC-GANs. We use an AutoEncoder to learn the embeddings on the latent manifold from real data.
		Relying on LCC, we learn a set of bases such that the LCC sampling can be conducted. As a result, the proposed method is able to take the constructed LCC codings to generate new data. }
		\label{Fig: AE_LCC_GANs}
	}
\end{figure*}

\section{Adversarial Learning with LCC}
In this section, we seek to improve GANs by exploiting LCC. The overall structure of the proposed method, called LCC-GANs, is illustrated in Figure \ref{Fig: AE_LCC_GANs}.

As shown in Figure \ref{Fig: AE_LCC_GANs}, instead of sampling from some pre-defined prior distribution, we seek to sample points from a learned latent manifold for training GANs.
Specifically,
we use an AutoEncoder (AE) to learn embeddings over a latent manifold of real data and then employ LCC to learn a set of bases to form local coordinate systems on the latent manifold. After that, we introduce LCC into GANs by approximating the generator using a linear function w.r.t. a set of codings (see Section \ref{sebsec:G_appro}).
Relying on such approximation, we then propose an LCC based sampling method to exploit the local information of data on the latent manifold (see Section \ref{subsec:lcc_sampling}).
The details of the proposed method are illustrated in following subsections.

\subsection{Generator Approximation Based on LCC} \label{sebsec:G_appro}
According to Definition \ref{definition: manifold}, any point on the latent manifold can be approximated by a linear combination of a set of local bases.
Inspired by this, if the bases are sufficiently localized, the generator of GANs can also be approximated by a linear function w.r.t. a set of codings.

\begin{figure*}[t]
	\centering
	{
		\includegraphics[width=0.56\linewidth]{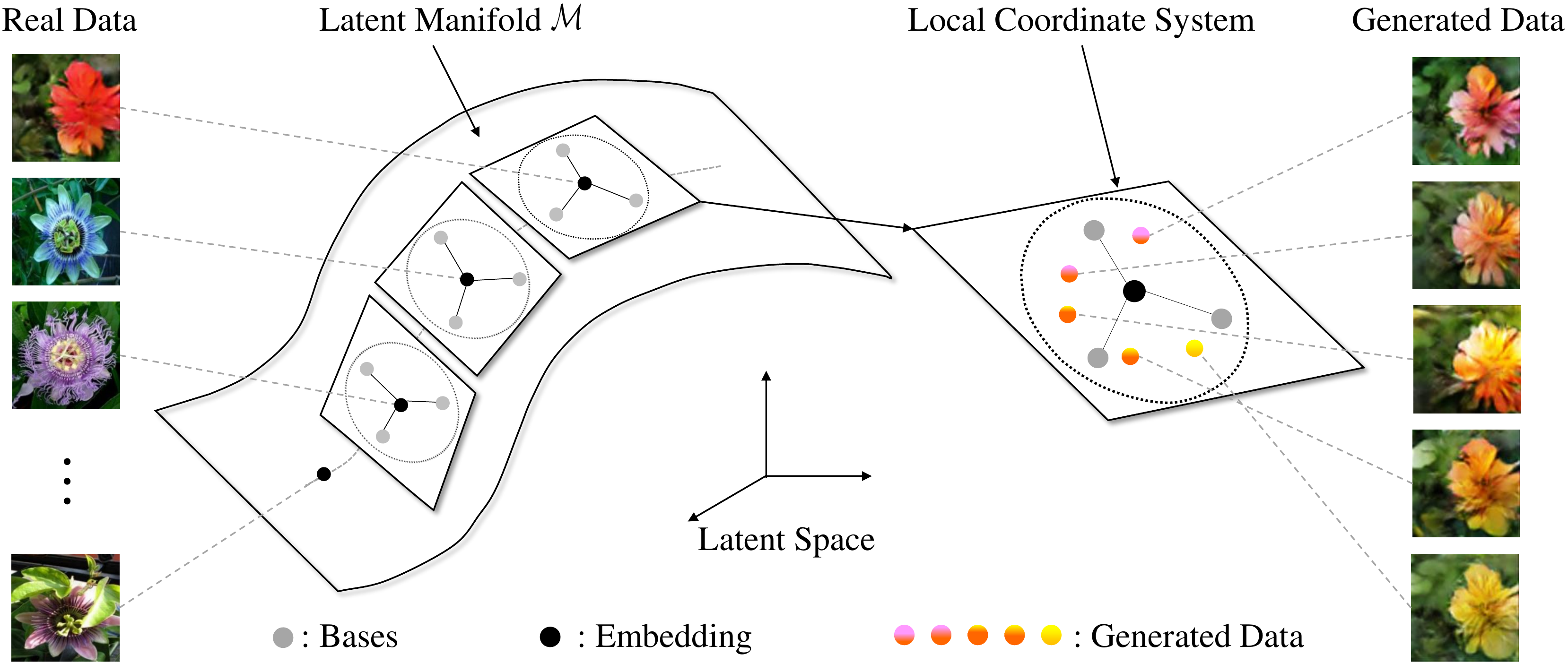}
		\caption{The geometric views on LCC Sampling. By learning embeddings (\textit{i.e.}, black points) which lie on the latent manifold, we use LCC to learn a set of bases (\textit{i.e.}, gray points) to form a local coordinate system such that we can sample different latent points (\textit{i.e.}, coloured points) by LCC sampling. As a result, LCC-GANs can generate new data which have different attributes. }
		\label{Fig: LCC_Sampling}
	}
\end{figure*}

\begin{lemma} ~\textbf{\emph{(Generator Approximation)}} \label{lemma: Generator Approximation}
	Let $ (\bgamma, \mC) $ be an arbitrary coordinate coding on $ \mmR^{d_B} $.
	Given a $ (L_{\bh}, L_{G}) $-Lipschitz smooth generator $ G_u(\bh) $, for all $ \bh \in \mmR^{d_B} $:
\begin{equation}\label{eqn:ga}
\begin{aligned}
	&\left\| G_u\left(\sum\nolimits_{\bv \in \mC} \gamma_{\bv}(\bh) \bv\right) \small{-} \sum\nolimits_{\bv \in \mC} \gamma_{\bv} (\bh) G_u(\bv) \right\|_2  \\ 
	\leq& 2L_{\bh} \| \bh \small{-} \br(\bh) \|_2 \small{+} L_G \sum\nolimits_{\bv \in \mC} |\gamma_{\bv} (\bh)| \small{\cdot} \| \bv \small{-} \br(\bh) \|_2^{2},
\end{aligned}
\end{equation}
	where $ \br(\bh) = \sum_{\bv \in \mC} \gamma_{\bv} (\bh) \bv $.
\end{lemma}

Given the local bases and a Lipschitz smooth generator, the generator w.r.t. the linear combination of the local bases can be approximated by the linear combination of the generator w.r.t. local bases.
Since two close latent points often share the same local bases but with different weights (\textit{i.e.}, codings), we can change these weights for generator approximation.
Therefore, the pieces of generated data can cover an entire manifold seamlessly (see Figure \ref{fig:LCC_b}).

\textbf{Objective function of LCC.} We minimize the right-hand term of the inequality in (\ref{eqn:ga}) to obtain a set of bases.
Given a set of the latent points $ \{ \bh_i \}_{i=1}^N $, by assuming $ \bh \approx \br(\bh) $ \cite{yu2009nonlinear},   we address the following problem:
\begin{equation}
\begin{aligned}
\small{\min_{\bgamma, \mC}}& \small{\sum_{\bh}} 2L_{\bh} \| \bh - \br(\bh) \|_2 \small{+} L_G \small{\sum\limits_{\bv \in \mC}} |\gamma_{\bv} (\bh)| \small{\cdot}\| \bv - \bh \|_2^{2} \\
\st& \sum_{\bv \in \mC} \gamma_{\bv} (\bh) = 1, \; \forall \,\bh, \label{problem: lcc coding}
\end{aligned}
\end{equation}
where $ \br(\bh) = \sum_{\bv \in \mC} \gamma_{\bv} (\bh) \bv $.
In practice, we update $ \bgamma $ and $ \mC $ by alternately optimizing a LASSO problem and a least-square regression problem, respectively.

\subsection{Objective Function of LCC-GANs}
After solving Problem (\ref{problem: lcc coding}), every latent point $ \bh \in \mmR^{d_B} $ is close to its physical approximation $ \br(\bh) $, \textit{i.e.}, $ \bh \approx \br(\bh) $, then the generator can be approximated by
\begin{align}\label{eqn: generator}
G_u(\bh) \approx G_u(\br(\bh)) \triangleq G_{w} (\bgamma(\bh)), \bh \in \mH,
\end{align}
where $ \br(\bh) = \bV \bgamma(\bh) $, $ \bV = \left[ \bv_1, \bv_2, \ldots, \bv_M \right] $ and $ \bgamma(\bh) = \left[ \gamma_1(\bh), \gamma_2(\bh), \ldots, \gamma_M(\bh) \right]^{\trsp} $ with $ M = |\mC| $. Here, $ \mH $ is the latent distribution and $ w \in \mW $ is the parameters of the generator w.r.t. $ u $ and fixed $ \bV $ learned from Problem (\ref{problem: lcc coding}).

Using the neural network distance, we consider the following objective function of LCC-GANs between the generated distribution and the empirical distribution:
\begin{align}\label{problem: lcc-gan}
\min_{G_w \in \mG} \; d_{\mF, \phi} \left(\widehat\mD_{G_{w} (\bgamma(\bh))}, \widehat\mD_{real} \right), \bh \in \mH.
\end{align}

To be more specific, Problem (\ref{problem: lcc-gan}) can be rewritten as:
\begin{align*}
\mathop {\min}_{\small{w \in \mW}} \mathop {\max}\limits_{v \in \mV}
\mathop\mmE_{\bx \sim \widehat{\mD}_{real}} \big[ \phi (D_v(\bx)) \big]
\small{+} \mathop{\mmE}_{\bh \sim \mH} \big[ \phi \big( \widetilde{D}_v\left(G_w\left(\bgamma(\bh) \right) \right) \big) \big],
\end{align*}
where $ \phi(\cdot) $ is a monotone function, and $ \widetilde{D}_v(\cdot) = 1 - D_v(\cdot) $.
The detailed algorithm is shown in Algorithm \ref{alg:lccgan}.

\subsection{LCC Sampling Method} \label{subsec:lcc_sampling}
To address Problem (\ref{problem: lcc-gan}), one of the key issues is on how to conduct sampling from the learned latent manifold.
Although the latent manifold can be learned by AutoEncoder, it is very hard to sample valid points on it to train GANs.
To address this,
we propose an LCC sampling method to capture the latent distribution on the learned latent manifold (see Figure~\ref{Fig: LCC_Sampling}). The proposed sampling method contains the following two steps.

\textbf{Step 1:}
Given a local coordinate system, we randomly select a latent point (specifically, it can be a basis), and then find its $ d $-nearest neighbors $ \mB = \{ \bv_j\}_{j=1}^d $.

\textbf{Step 2:}
We construct an $M$-dimensional vector $ \bgamma(\bh) = [ \gamma_1(\bh), \gamma_2(\bh), \ldots, \gamma_M(\bh) ]^{\trsp} $ as the LCC coding for sampling.
Here, each element of $ \bgamma(\bh) $ is corresponding to the weight of the basis.
To conduct local sampling, we construct the coding of the neighbors $ \mB $ as follows:
\begin{align*}
\gamma_{j} (\bh) =
\left\{ \begin{array}{l}
{z_j},\;\;{\bv_j} \in {\mB}\\
\; 0\,,\;\;{\bv_j} \notin {\mB}
\end{array} \right.,
\end{align*}
where $ z_j $ is the $ j $-th element of $ \bz \in \mmR^d $ from the prior distribution $ p(\bz) $.
Here, we set $ p(\bz) $ to be the standard Gaussian distribution $ \mN(\0, \bI) $.
Finally, we obtain a new latent point $ \bV \bgamma(\bh) $.

Based on Definition \ref{definition: manifold}, the intrinsic dimensionality is determined by the number of bases in a local region. Thus, we turn the determination of intrinsic dimension into an easier problem of selecting sufficient number of local bases.

\begin{algorithm}[t]
	\caption{LCC-GANs Training Method.}
	\label{alg:lccgan}
	\begin{algorithmic}[1]\small
		\REQUIRE  Training data $\{ \bx_i \}_{i=1}^N $; a prior distribution $p(\bz)$, where $\bz \in \mmR^d  $; minibatch size $ n $. \\
		\STATE Learn the latent manifold $\mM$ using an AutoEncoder \\
		\STATE Construct LCC bases $\{ \bv_i \}_{i=1}^M$ on $\mH$ by optimizing: \\
		\vspace{5pt}
		~~$ \small{\min_{\bgamma, \mC}}~ \small{\sum_{\bh}} 2L_{\bh} \| \bh - \br(\bh) \|_2 + L_G \small{\sum_{\bv \in \mC}} |\gamma_{\bv} (\bh)| \small{\cdot} \| \bv - \bh \|_2^{2}$ \\
		\vspace{5pt}
		\FOR{number of training iterations}
		\STATE Do LCC Sampling to obtain a minibatch $\{ \gamma(\bh_i)\}_{i=1}^n$ \\ 
		\STATE Sample a minibatch $\{\bx_i\}_{i=1}^n$ from the data distribution \\
		\STATE Update the discriminator by ascending the gradient: \\
		\vspace{3pt}
		~~~~~~~$\nabla_v \frac{1}{n} \sum\nolimits_{i=1}^n \phi(D_v(\bx_i)) + \phi( (1-D_v(G_w(\gamma(\bh_i)))))$\\
		\vspace{3pt}
		\STATE Do LCC Sampling to obtain a minibatch $\{ \gamma(\bh_i)\}_{i=1}^n$ \\
		\STATE Update the generator by descending the gradient: \\
		\vspace{3pt}
		~~~~~~~~~~~~~~~~~~$\nabla_w \frac{1}{n} \sum\nolimits_{i=1}^n \phi (1-D_v(G_w(\gamma(\bh_i))))$\\
		\ENDFOR
	\end{algorithmic}
\end{algorithm}

\section{Theoretical Analysis} \label{section: Theoretical Analysis}
We first give some necessary notations to develop our theoretical analysis for LCC based GANs.
Let $\{ \bx_i \}_{i=1}^N $ be a set of observed training samples drawn from the real distribution $ \mD_{real} $, and let $ \widehat{\mD}_{real} $ denote the empirical distribution over $\{ \bx_i \}_{i=1}^N $.  Given a generator $ G_u $ and a set of the latent points $ \{ \bh_i \}_{i=1}^r $, $ \{G_u (\bh_i)\}_{i=1}^r$ denotes a set of $ r $ generated samples from the generated distribution $ \mD_{G_u} $, and $ \widehat{\mD}_{G_{w}} $ is an empirical generated distribution.
Motivated by \cite{arora2017gans, zhang2018on}, we define the generalization of LCC-GANs as follows:

\begin{deftn} \textbf{\emph{(Generalization) }}\label{definition: generalization}
	The neural network distance $ d_{\mF, \phi}(\cdot, \cdot) $ between distributions generalizes with $ N $ training samples and error $ \epsilon $, if for a learned distribution $ \mD_{G_{u}} $, the following holds with high probability,
	\begin{align*}
	\left|d_{\mF, \phi} \left(\widehat{\mD}_{G_{w}}, {\mD}_{real} \right) - \inf_{\mG } d_{\mF, \phi} \left({\mD}_{G_u},  {\mD}_{real} \right) \right| \leq \epsilon.
	\end{align*}
\end{deftn}

In Definition \ref{definition: generalization}, the generalization of GANs means that the population distance $ d_{\mF, \phi} ({\mD}_{G_u},  {\mD}_{real} ) $ is close to the distance $ d_{\mF, \phi} (\widehat{\mD}_{G_{w}}, {\mD}_{real} ) $. In theory, we hope to obtain a small $ d_{\mF, \phi} ({\mD}_{G_u},  {\mD}_{real} ) $.  In practice, we can minimize the empirical loss $ d_{\mF, \phi} (\widehat{\mD}_{G_{{w}}}, \widehat{\mD}_{real} ) $ to approximate  $ d_{\mF, \phi} (\widehat{\mD}_{G_{{w}}}, {\mD}_{real} ) $.
First, we have the following generalization bound on $ \widehat{\mD}_{real} $ to develop the generalization analysis of LCC-GANs.

\begin{thm} \label{theorem: Generalization Bound}
	Suppose $ \phi(\cdot) $ is Lipschitz smooth: $ | \phi'(\cdot) | \leq L_{\phi} $, and bounded in $ [-\Delta, \Delta] $. Given the coordinate coding $ (\bgamma, \mC) $, an example set $ \mH $ in latent space and the empirical distribution $ \widehat{\mD}_{real} $, if the generator is Lipschitz smooth
	, then the expected generalization error satisfies:
	\begin{align*}
	&\mmE_{\mH} \left[ d_{\mF, \phi} \left(\widehat{\mD}_{G_{\widehat{w}}\left( \bgamma(\bh) \right)}, \widehat{\mD}_{real} \right)
	\right] \\
	\leq& \inf_{ \mG } \mmE_{\mH} \left[ d_{\mF, \phi} \left( {\mD}_{ G_{u} (\bh)}, \widehat{\mD}_{real} \right) \right] + \epsilon(d_{\mM}),
	\end{align*}
where $ \epsilon(d_{\mM}) = L_{\phi} Q_{L_{\bh}, L_{G}} (\bgamma, \mC) + 2\Delta $, and generative quality $ Q_{L_{\bh}, L_{G}} (\bgamma, \mC) $ is bounded w.r.t. $ d_{\mM} $ in Lemma 3 which is given in supplementary materials. 
\end{thm}
See supplementary materials for the proof.

Theorem \ref{theorem: Generalization Bound} shows that the generalization bound for $ \widehat{\mD}_{real} $ is related to the dimension of the latent manifold (\textit{i.e.}, $ d_{\mM} $)  rather than the dimension of the latent space (\textit{i.e.}, $ d_{B} $).
Based on Theorem \ref{theorem: Generalization Bound} and the Rademacher complexity \cite{bartlett2002rademacher}, we then accomplish the generalization bound on an unknown real distribution $ \mD_{real} $.
\begin{thm} \label{theorem: generalization_Rademacher}
Under the condition of Theorem \ref{theorem: Generalization Bound}, given an empirical distribution $ \widehat{\mD}_{real} $ drawn from $ \mD_{real} $, the following holds with probability at least $ 1 - \delta $,
	\begin{align*}
		& \left| \mmE_{\mH} \left[d_{\mF, \phi} \left(\widehat{\mD}_{G_{\widehat{w}}}, {\mD}_{real} \right) \right]
		\small{-} \inf_{ \mG } \mmE_{\mH} \left[ d_{\mF, \phi} \left({\mD}_{G_u},  {\mD}_{real} \right) \right] \right| \\
		\leq& 2 {R}_{\mX}(\mF) + 2 \Delta \sqrt{\frac{2}{N} \log(\frac{1}{\delta})} + 2\epsilon(d_{\mM}),
	\end{align*}
	where $ {R}_{\mX}(\mF) $ is the Rademacher complexity of $\mF $.
\end{thm}
See supplementary materials for the proof.

Theorem \ref{theorem: generalization_Rademacher} shows that the generalization error of LCC-GANs can be bounded by Rademacher complexity of $\mF $ and an error term $ \epsilon(d_{\mM}) $.
Specifically, 
the former term $ {R}_{\mX}(\mF) $ implies that the set of discriminator should be smaller to have better generalization ability, and also be large enough to be able to identify the data distribution, which is consistent with \cite{zhang2018on}.
The latter term $ \epsilon(d_{\mM}) $ indicates that a small dimensional input is sufficient to achieve good generalization.
In practice, every dataset has its own dimension of the latent manifold. Nevertheless, experiments show that the proposed method is able to generate perceptually convincing images with small dimensional inputs.

\section{Experiments}
We compare LCC-GANs with several state-of-the-arts, namely Vanilla GANs~\cite{radford2015unsupervised}, WGANs~\cite{arjovsky2017wasserstein} and Progressive GANs~\cite{karras2017progressive}.
Here, Vanilla GANs and Progressive GANs are used to implement our LCC-GANs.
For all considered GAN methods, the inputs are sampled from a $ d $-dimensional prior distribution, and we train the generative models to produce $64 \times 64$ images.
All  experiments are conducted on a single Nvidia Titan X GPU.

\noindent \emph{\textbf{Implementation details}}.
We implement LCC-GANs based on PyTorch.\footnote{PyTorch is from http://pytorch.org/.} We follow the experimental settings in DCGANs~\cite{radford2015unsupervised}. Specifically, for the optimization, we use Adam~\cite{kingma2014adam} with a mini-batch size of 64 and a learning rate of 0.0002 to train the generator and the discriminator.
We initialize the parameters of both the generator and the discriminator following the strategy in~\cite{he2015delving}.

\noindent \emph{\textbf{Datasets and evaluation metrics}}.
To thoroughly evaluate the proposed method, we conduct experiments on a wide variety of benchmark datasets, including MNIST~\cite{lecun1998gradient}, Oxford-102~\cite{nilsback2008automated}, LSUN~\cite{song2015construction} and CelebA~\cite{liu2015deep}.
For quantitative comparisons, we adopt the \emph{Inception Score} (IS)~\cite{salimans2016improved} and \emph{Multi-Scale Structural Similarity }(MS-SSIM)~\cite{karras2017progressive} as the performance metrics, which are highly consistent with human evaluations. Inception Score measures both the single image quality and the diversity over a large number of samples (\textit{i.e.}, 50k).
In general, a larger IS value corresponds to the better performance of the method, and a smaller MS-SSIM value corresponds to images with more diversity.

\begin{figure}[h]
	\centering
	\subfigure[Generated samples with $d=3$. The yellow and red boxes denote similar generated digits ``2'' and ``8'', respectively.
	]{
		\label{fig:mnist_low}
		\includegraphics[width = 1\columnwidth]{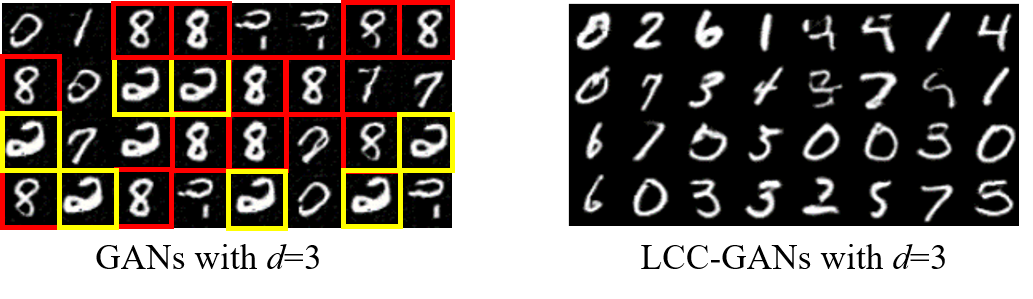}\label{fig:ratio}
	}
	\subfigure[Comparions of different GANs with $d=5$, where GANs with $d=100$ are considered as the baseline.]{
		\label{fig:mnist_high}
		\includegraphics[width = 1\columnwidth]{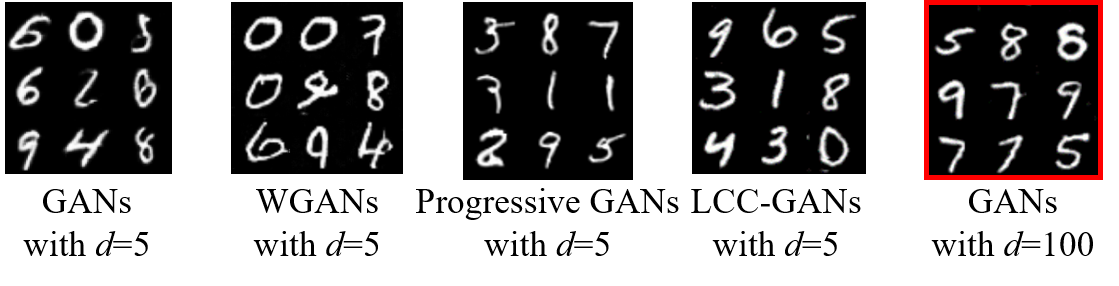}\label{fig:ratio_testing_error}
	}
	\caption{Performance comparisons of various GANs on MNIST.}
	\label{fig:mnist}
\end{figure}

\begin{table*}[htp]
	\small
	\centering
	\caption{Inception scores of various generative models on Oxford-102. For each method, we produce $50,000$ samples for testing.}
	\begin{tabular}{c|ccccc}
		\hline
		Samples & \includegraphics[width=0.3\columnwidth]{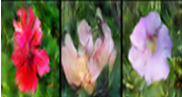}
		&  \includegraphics[width=0.3\columnwidth]{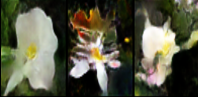}
		&  \includegraphics[width=0.3\columnwidth]{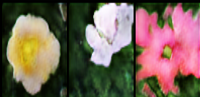}
		&  \includegraphics[width=0.3\columnwidth]{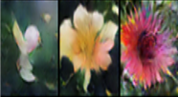}
		&  \includegraphics[width=0.3\columnwidth]{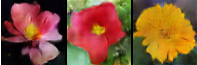} \\
		\hline
		Method & \multicolumn{1}{c}{GANs ($d$=10)} & \multicolumn{1}{c}{WGANs ($d$=10)} & \multicolumn{1}{c}{Progressive GANs ($d$=10)} & \multicolumn{1}{c}{GANs ($d$=100)} & \multicolumn{1}{c}{LCC-GANs ($d$=10)} \\
		\hline
		Scores & $2.21 \pm 0.03$    &    $2.14 \pm 0.02$   &   $2.43 \pm 0.05$   &  2.66 $\pm$ 0.03   &    \textbf{$2.71 \pm 0.03$}     \\
		\hline
	\end{tabular}
	\label{tab:flowers}
\end{table*}

\begin{figure}[t]
	\centering
	\includegraphics[width=1\columnwidth]{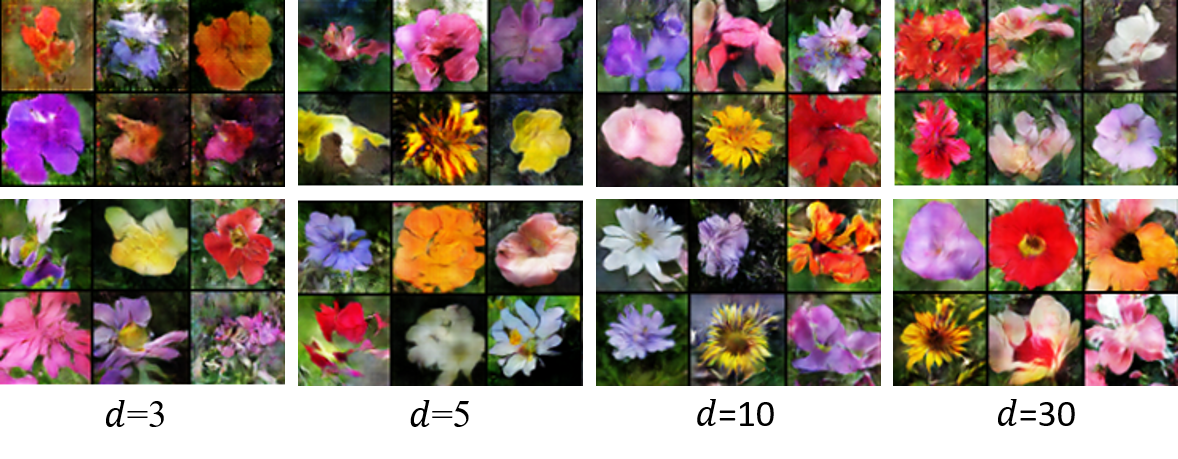}
	\caption{Results of LCC-GANs and Vanilla GANs on Oxford-102. Top: Vanilla GANs. Bottom: LCC-GANs.}
	\label{fig:flower-result}
\end{figure}

\subsection{Results on MNIST}
In this experiment, we evaluate the performance of the proposed method on MNIST~\cite{lecun1998gradient}, which contains handwritten digit images ranging from 0 to 9.
In this small dataset, we adopt Vanilla GANs as the baseline to implement the proposed LCC-GANs.
The visual comparisons are shown in Figure~\ref{fig:mnist}.

From Figure~\ref{fig:mnist_low},
given a very low dimensional input with $d=3$, Vanilla GANs  produce only few kinds of digits with almost the same shapes (see the yellow and red boxes in Figure~\ref{fig:mnist_low}.
In other words, Vanilla GANs produce images with very low diversity. 
In contrast, LCC-GANs with a small dimensional input $d=3$ can produce digits with different styles and different orientations. Equipped with LCC, the proposed LCC-GANs effectively preserve the local information of data on the latent manifold and thus help the training of GANs.

In Figure~\ref{fig:mnist_high}, we increase the dimension of input to $d=5$ and compare the proposed LCC-GANs with other state-of-the-art GAN methods.
In this experiment, the baseline GAN methods often produce
digits with obscure structure.
Nevertheless, the proposed LCC-GANs significantly outperform the considered baseline methods and produce sharp images with high diversity.
More critically, LCC-GANs with $d=5$ are able to achieve comparable or even better performance than their GAN counterparts with $d=100$ (see red box in Figure~\ref{fig:mnist_high}).
These results show the efficacy of the proposed LCC-GANs when training a generative model with the local information of the latent manifold.
Compared to the baseline methods, LCC-GANs only need a relatively low dimensional input to produce visually promising images.

\begin{table}[t]
	\small
	\centering
	\caption{Inception-Score (IS) and MS-SSIM on Oxford-102.}
	\resizebox{0.48\textwidth}{!}{
		\begin{tabular}{c|c|c|c|c|c|c|c|c}	
			\hline 
			\multicolumn{1}{c|}{\multirow{2}[0]{*}{Methods}} & \multicolumn{2}{c|}{$ d=5 $} & \multicolumn{2}{c|}{$ d=10 $} & \multicolumn{2}{c|}{$ d=30 $} & \multicolumn{2}{c}{$ d=100 $}\\
			\cline{2-9}
			& \multicolumn{1}{c|}{IS} & \multicolumn{1}{c|}{SSIM}
			& \multicolumn{1}{c|}{IS} & \multicolumn{1}{c|}{SSIM}
			& \multicolumn{1}{c|}{IS} & \multicolumn{1}{c|}{SSIM}
			& \multicolumn{1}{c|}{IS} & \multicolumn{1}{c}{SSIM} \\
			\hline 
			GANs           & 2.03 & 0.205 & 2.37 & 0.180 & 2.57 & 0.166 & 2.66 & 0.160 \\
			VAE           & 2.14 & 0.203 & 2.38 & 0.185 & 2.54 & 0.163 & 2.68 & 0.162 \\
			Sparse Coding & 2.44 & 0.197 & 2.63 & 0.179 & 2.68 & 0.157 & 2.72 & 0.153 \\
			\hline 
			LCC Coding    & 2.57 & 0.188 & 2.71 & 0.163 & 2.83 & 0.153 & 2.75 & 0.147 \\
			\hline 
		\end{tabular}
	}
	\label{table:representation}
\end{table}
\subsection{Results on Oxford-102 Flowers}
We further evaluate the proposed LCC-GANs on a larger dataset Oxford-102 which contains flower images of 102 categories. 
In this experiment, we adjust the input of generative models with different dimensions, \textit{i.e.}, $d=\{3, 5, 10, 30\}$, and adopt Vanilla GANs to implement the proposed LCC-GANs and investigate the effect of different input dimensions.
The results are shown in Figure~\ref{fig:flower-result}.

From Figure~\ref{fig:flower-result}, we have the following observations. First,
for Vanilla GANs, the performance highly depends on the input dimension. Given a small dimension, \textit{i.e.}, $d=3$ or $d=5$, the GAN models often fail to produce meaningful flowers and obtain images with a blurring structure and distorted regions.
In contrast, LCC-GANs can produce promising images with clear structure given an input with $d=5$.
With such a low dimensional input, LCC-GANs effectively capture the local information of the latent manifold and produce perceptually convincing images.
Second, we further investigate the effect of input dimension. From Figure~\ref{fig:flower-result}, the proposed LCC-GANs consistently outperform their baseline GAN methods given the inputs of different dimensions.

Moreover, we compare the proposed LCC-GANs with several state-of-the-art GAN methods and report the results in Table~\ref{tab:flowers}.
From Table~\ref{tab:flowers}, the proposed LCC-GANs with $d=10$ significantly outperform the other baseline methods and achieve the best performance with a score of 2.71. More critically, LCC-GANs with $d=10$ achieve even better performance than Vanilla GANs with $d=100$, which require the input with much higher dimension.

\textbf{Comparisons of different representation methods. }
On Oxford-102, we compare different representation methods and adopt Inception Score and MS-SSIM to evaluate the quality and diversity of the generated images, respectively.
We adjust the input with different dimensions, \textit{i.e.}, $ d = \{ 5, 10, 30, 100 \} $, and adopt Vanilla GANs to implement LCC-GANs.
The results are shown in Table \ref{table:representation}.

From Table \ref{table:representation}, LCC-GANs consistently outperform other methods with various $ d $ in both measures.
These results show the effectiveness of the proposed LCC-GANs in producing perceptually promising images with higher quality and larger diversity than the considered baselines.

\begin{figure}[t]
	\centering
	\subfigure[Results of LCC-GANs with $d=10$.]{
		\includegraphics[width=1\linewidth]{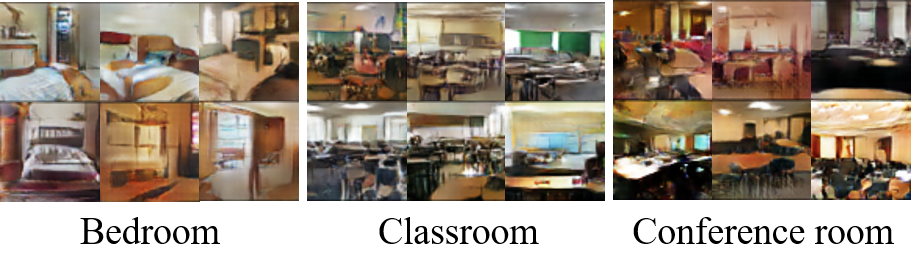}
	}
	\subfigure[Results of Vanilla GANs with $d=10$.]{
		\includegraphics[width=1\linewidth]{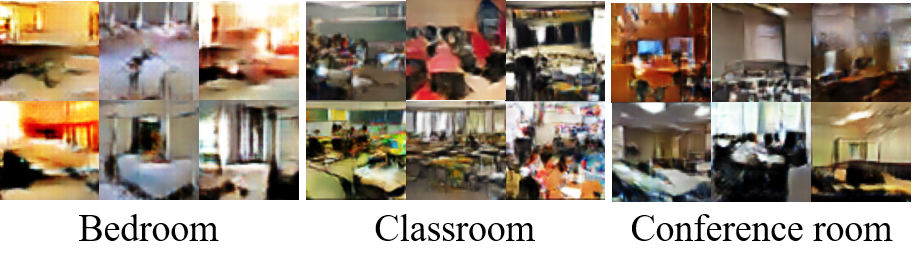}
	}
	\subfigure[Results of Vanilla GANs with $d=100$.]{
		\includegraphics[width=1\linewidth]{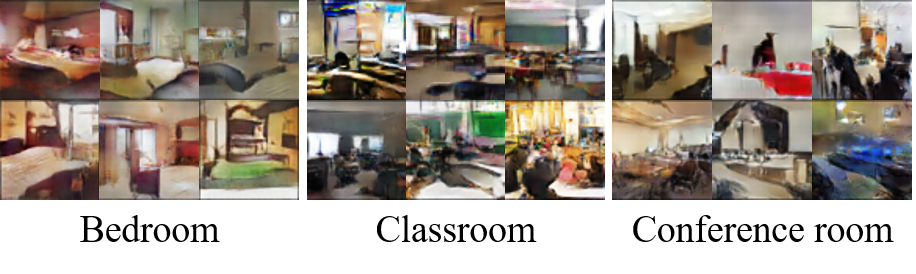}
	}
	\caption{Results of LCC-GANs with Vanilla GANs for different dimensions of the latent distribution on LSUN. }
	\label{fig:lsun}
\end{figure}

\subsection{Results on LSUN}
In this experiment, we evaluate the proposed LCC-GANs on LSUN which is a collection of natural images of indoor scenes. We train the generative models to produce images of different categories, including bedroom, classroom, conference room, etc. In this experiment, we also adopt Vanilla GANs as the baseline models to implement LCC-GANs. We show the visual comparison results in Figure~\ref{fig:lsun}.

\begin{figure}[t]
	\centering
	\subfigure[Results of LCC-GANs with $d = 30$.]{
		\includegraphics[width=1\linewidth]{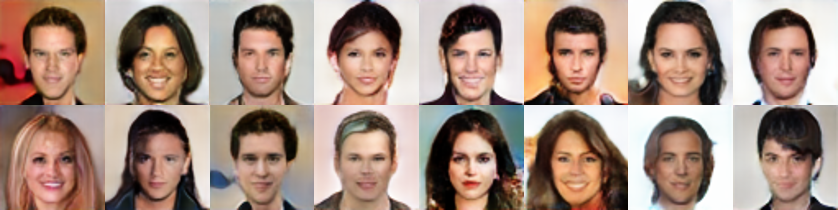}
	}\vspace{10pt}
	\subfigure[Results of Progressive GANs with $d = 30$.]{
		\includegraphics[width=1\linewidth]{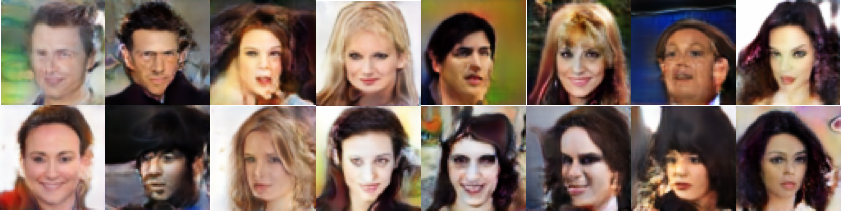}
	}\vspace{10pt}
	\subfigure[Results of Progressive GANs with $d = 100$.]{
		\includegraphics[width=1\linewidth]{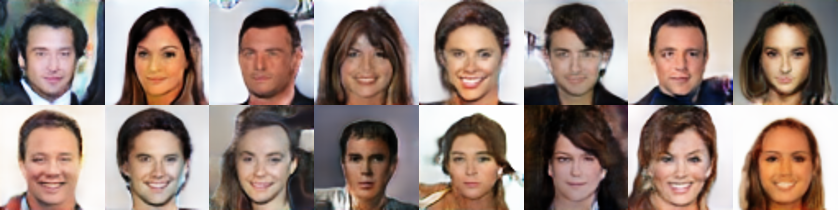}
	}
	\caption{Performance comparisons of LCC-GANs with Progressive GANs. }\label{fig:celeba-lccgan}
\end{figure}

From Figure~\ref{fig:lsun}, when we train the models using an input with a small dimension $d=10$, Vanilla GANs often fail to generate clear and meaningful images. In contrast, LCC-GANs significantly outperform their GANs counterparts and produce images with sharp structure and rich details. Moreover, when generating images of different scenes, LCC-GANs consistently outperform Vanilla GANs.
Note that the scene images in LSUN are much more complex than the images of MNIST and Oxford-102. Therefore, training a generative model can be more difficult. However, with the help of LCC, the proposed LCC-GANs are able to effectively capture the local common features and produce visually convincing images.

In this experiment, we also present the generated samples of Vanilla GANs with a high dimensional input $d=100$. Compared to this method, LCC-GANs only require an input with $d=10$ to produce even better images. In other words, this LCC sampled input effectively preserves the local information of real images on the latent manifold and thus helps the training of GANs.
With the help of LCC sampling, most of the generated images show sharper structure and contain more meaningful details.

\subsection{Results on CelebA}
In this experiment, we evaluate the proposed method on the large-scale dataset CelebA, which is composed of a set of celebrity faces. Here, Progressive GANs \cite{karras2017progressive} are adopted to implement LCC-GANs.
We conduct comparisons and show the results in Figure~\ref{fig:celeba-lccgan}.

Since face images often share a common face outline and only differ in detailed attributes, \textit{e.g.}, hair, eyes, mouth, skin features, it requires an input with a larger dimension to capture the local information.
In this way, we adopt the input with a larger dimension for both Progressive GANs and the proposed LCC-GANs in the training.
From Figure~\ref{fig:celeba-lccgan}, the performance of Progressive GANs degrades severely given an input with a small dimension $d=30$, compared to $d=100$.
However, with the help of LCC coding, the proposed LCC-GANs with the input of $d=30$ are able to produce images of better quality than Progressive GANs with high dimensional inputs of $d=100$.
According to these results, LCC-GANs greatly benefit from the LCC sampling and make the training much easier than directly matching the standard Gaussian distribution.

\begin{table}[t]
	\caption{Generated images from LCC sampling on MNIST, Oxford-102 and CelebA. The last column shows the most similar images in training set to the generated samples on the left.}\label{fig:lcc-sampling}
	\centering
	\includegraphics[width=1\linewidth]{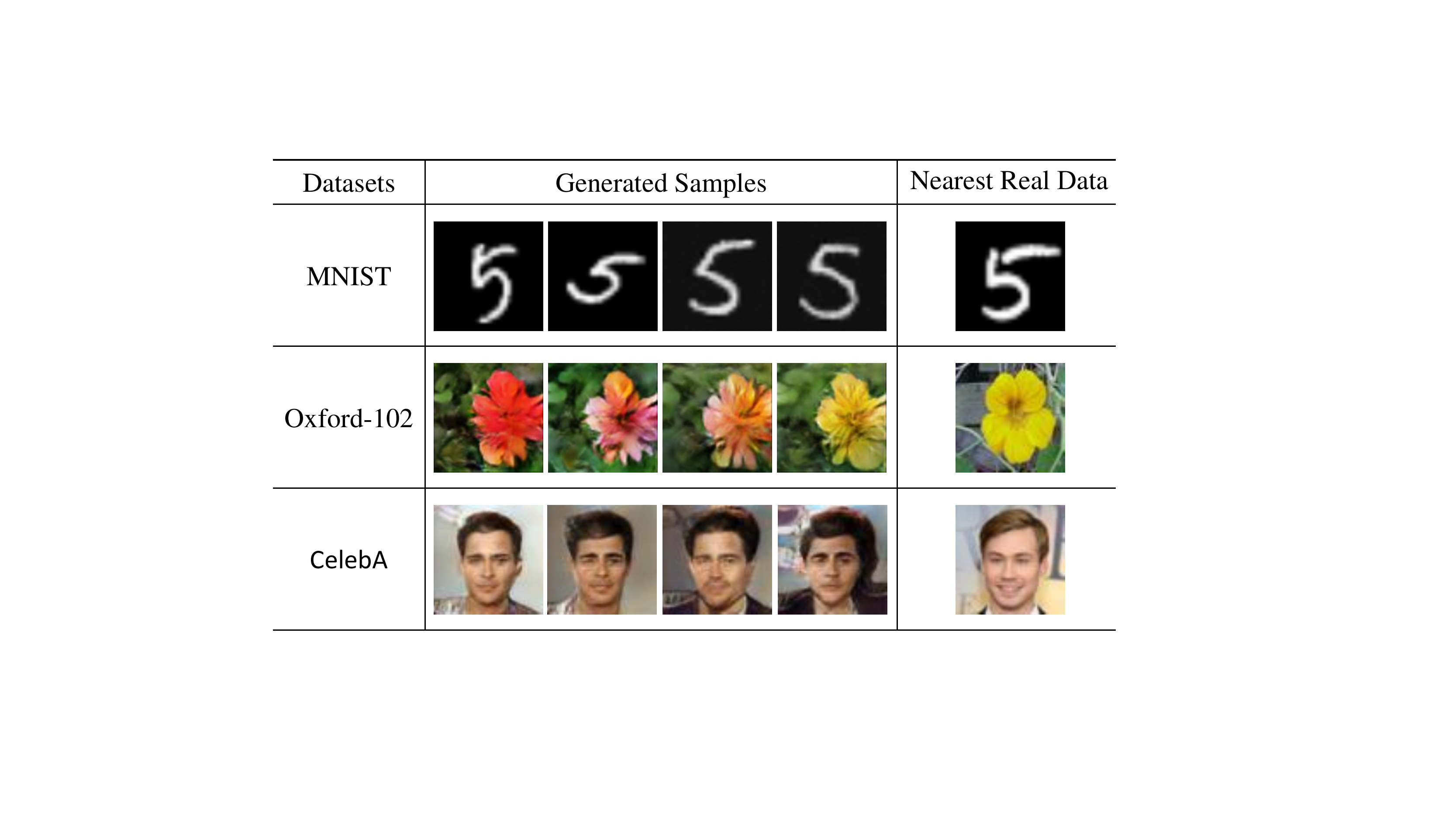} 
\end{table}

\subsection{Demonstration of LCC Sampling}
In this experiment, we investigate the effectiveness of the proposed LCC sampling method.
To achieve this, we can simply fix a specific set of bases and only change the corresponding weights to produce images. Ideally, these images should be located in a local area of the latent manifold and share some common features.

We conduct experiments on LCC sampling and show the results
in Table~\ref{fig:lcc-sampling}. 
The second column of Table~\ref{fig:lcc-sampling} shows the generated images sampled by LCC sampling method on different datasets.
The last column is the real image with the largest similarity to the generated images.
From Table~\ref{fig:lcc-sampling}, 
LCC-GANs produce digits with sharp shapes and different orientations or styles (see the top row in Table~\ref{fig:lcc-sampling}).
Each  generated image contains a digit ``5'' but with obvious individual differences.
In other words, the proposed LCC sampling method is able to generate new data by effectively exploiting the local information on the latent manifold.

When synthesizing flowers and faces, we draw a similar conclusion that verifies the effectiveness of the proposed LCC sampling method.
Specifically, LCC-GANs produce flowers with similar shapes but with different colors. Similarly, LCC-GANs also produce varying face images of promising quality which share some common features. 
These results demonstrate that the proposed LCC-GANs generalize well to unseen data rather than simply memorizing the training samples.

\subsection{More Quantitative Results}
In this experiment, we adopt MS-SSIM as the evaluation measure and compare the proposed LCC-GANs with several GAN methods on four benchmark datasets.
We use Vanilla GANs to implement LCC-GANs. 
To show the superiority of the proposed method, we set $d=30$ for LCC-GANs and $d=100$ for the other baselines. 
Here, we can only report MS-SSIM because Inception Score is no longer a valid measure and may give misleading results on CelebA \cite{barratt2018note}.
The quantitative results are shown in Table \ref{tab:representation}.

From Table \ref{tab:representation}, with a low dimensional input, the proposed method is able to produce images with larger or comparable diversity (smaller MS-SSIM score) than the considered baselines with high dimensional inputs on most datasets.
These results show the effectiveness of the proposed LCC-GANs in generating images with large diversity.

\begin{table}[t]
	\centering
	\caption{MS-SSIM on different datasets. Here, $ d=100 $ is for all baselines and $ d=30 $ for LCC-GANs.}
	\resizebox{0.48\textwidth}{!}{
		\begin{tabular}{c|ccccc}
			\hline 
			Methods         & MNIST & Oxford-102 & LSUN  & CelebA \\
			\hline 
			GANs             & 0.242 & 0.160      & 0.224 & 0.337 \\
			WGANs            & 0.251 & 0.157      & 0.237 & 0.324 \\
			Progressive GANs & 0.239 & 0.151      & 0.213 & 0.308 \\
			\hline
			LCC-GANs         & 0.224 & 0.153      & 0.203 & 0.305 \\
			\hline 
		\end{tabular}
	}
	\label{tab:representation}
\end{table}

\section{Conclusion}
In this paper, we have proposed a novel generative model by exploiting
the local information on the latent manifold of real data to improve GANs using Local Coordinate Coding (LCC). 
Unlike existing methods, based on a generator approximation, we have developed an LCC based sampling method to train GANs. 
In this way, we are able to conduct analysis on the generalization ability of GANs and theoretically prove that a small dimensional input will help to achieve good generalization.
Extensive experiments on several benchmark datasets demonstrate the superiority of the proposed method over the state-of-the-art methods. Specifically, with the proposed LCC sampling, the proposed method outperforms the considered baselines by producing sharper images with higher diversity.

\textbf{Acknowledgments\\}
This work was supported by National Natural Science Foundation of China (NSFC) 61502177 and 61602185, and Recruitment Program for Young Professionals, and Guangdong Provincial Scientific and Technological funds 2017B090901008, 2017A010101011, 2017B090910005, and Fundamental Research Funds for the Central Universities D2172500, D2172480, and Pearl River S\&T Nova Program of Guangzhou 201806010081 and CCF-Tencent Open Research Fund RAGR20170105.

\bibliographystyle{icml2018}

\onecolumn
\vskip 0.3in
\icmltitle{Supplementary Materials for\\ ``Adversarial Learning with Local Coordinate Coding''}

\begin{lemma} \label{lemma: DG Approximation}
	Let $ (\bgamma, \mC) $ be an arbitrary coordinate coding on $ \mmR^{d_B} $.
	Given an $ (L_{\bh}, L_{G}) $-Lipschitz smooth generator $ G_u(\bh) $ and an $ L_{\bx} $-Lipschitz discriminator $ D_{v} $, for all $ \bh \in \mmR^{d_B} $:
	\begin{align*}
		\left| {D}_v(G_u(\bh)) - {D}_v \left(\sum_{\bv} \gamma_{\bv} (\bh) G_u(\bv) \right) \right| 
		\leq L_{\bx} L_{\bh} \| \bh - \br(\bh) \|_2 + L_{\bx} L_G \sum_{\bv \in \mC} |\gamma_{\bv} (\bh)| \| \bv - \br(\bh) \|_2^{2}.
	\end{align*}
\end{lemma}
\begin{proof}
	Given an $ (L_{\bh}, L_{G}) $-Lipschitz smooth generator $ G_u(\bh) $, an $ L_{\bx} $-Lipschitz discriminator $ D_{v} $, and let $ \gamma_{\bv} = \gamma_{\bv}(\bh) $ and $ \bh' = \br(\bh) = \sum_{\bv \in \mC} \gamma_{\bv} \bv $. We have
	\begin{align*}
		&\left| \widetilde{D}_v(G_u(\bh)) - \widetilde{D}_v \left(\sum_{\bv} \gamma_{\bv} (\bh) G_u(\bv) \right) \right| \\
		=& \left| {D}_v(G_u(\bh)) - {D}_v \left(\sum_{\bv} \gamma_{\bv} (\bh) G_u(\bv) \right) \right| \\
		=& \left| {D}_v(G_u(\bh)) - {D}_v(G_u(\bh')) - \left( {D}_v \left(\sum_{\bv} \gamma_{\bv} (\bh) G_u(\bv) \right) - {D}_v(G_u(\bh')) \right)  \right| \\
		\leq& \left| {D}_v \left( G_u(\bh) \right) - {D}_v \left( G_u(\bh') \right) \right| + \left| {D}_v \left(\sum_{\bv} \gamma_{\bv} (\bh) G_u(\bv) \right) - {D}_v \left( G_u(\bh') \right)  \right| \\
		\leq& L_{\bx} \left\| G_u(\bh) - G_u(\bh') \right\|_2 + L_{\bx} \left\| \sum_{\bv} \gamma_{\bv} (\bh) G_u(\bv) - G_u(\bh') \right\|_2 \\
		\leq& L_{\bx} \left\| G_u(\bh) - G_u(\bh') \right\|_2 + L_{\bx} \left\| \sum_{\bv} \gamma_{\bv} (\bh) \left( G_u(\bv) - G_u(\bh') - \Delta G_u (\bh')^{\trsp} \left( \bv - \bh' \right) \right)  \right\|_2 \\
		\leq& L_{\bx} \left\| G_u(\bh) - G_u(\bh') \right\|_2 + L_{\bx} \sum_{\bv \in \mC} |\gamma_{\bv}| \left\| G_u (\bv) - G_u(\bh') - \Delta G_u(\bh')^{\trsp} (\bv - \bh') \right\|_2 \\
		\leq& L_{\bx} L_{\bh} \| \bh - \bh' \|_2 + L_{\bx} L_{G} \sum_{\bv \in \mC} |\gamma_{\bv} | \| \bv - \bh' \|_2^2 \\
		=& L_{\bx} L_{\bh} \| \bh -  \br(\bh) \|_2 + L_{\bx} L_{G} \sum_{\bv \in \mC} |\gamma_{\bv} | \| \bv - \br(\bh) \|_2^2,
	\end{align*}
	where $ \widetilde{D}_v(\cdot) = 1 - D_v(\cdot) $. In the above derivation, the first inequality holds by the triangle inequality. The second inequality uses an assumption that $ D_v $ is Lipschitz smooth w.r.t. the input. The third inequality uses the facts that $ \sum_{\bv \in \mC} \gamma_{\bv} (\bx) = 1 $ and $ \bh' = \sum_{\bv \in \mC} \gamma_{\bv} \bv $. The last inequality uses the $ (L_{\bh}, L_{G}) $-Lipschitz smooth generator $ G_u $, that is
	\begin{align*}
		\left\| G_u (\bv) - G_u(\bh') - \Delta G_u(\bh')^{\trsp} (\bv - \bh') \right\|_2 \leq L_{G} \| \bv - \bh' \|_2^2.
	\end{align*}
	This implies the bound.
\end{proof}

\newpage
\section{Proof of Lemma \ref{lemma: Generator Approximation}}
\begin{*lemma}{\emph{\textbf{\ref{lemma: Generator Approximation}} }} 
	\textbf{\emph{(Generator Approximation) }} 
	Let $ (\bgamma, \mC) $ be an arbitrary coordinate coding on $ \mmR^{d_B} $.
	Given a Lipschitz smooth generator $ G_u(\bh) $, for all $ \bh \in \mmR^{d_B} $:
	\begin{align*}
		\left\| G_u\left(\sum_{\bv \in \mC} \gamma_{\bv}(\bh) \bv\right) - \sum_{\bv \in \mC} \gamma_{\bv} (\bh) G_u(\bv) \right\|_2 
		\leq 2L_{\bh} \| \bh - \br(\bh) \|_2 + L_G \sum_{\bv \in \mC} |\gamma_{\bv} (\bh)| \| \bv - \br(\bh) \|_2^{2}.
	\end{align*}
\end{*lemma}

\begin{proof}
	From Lemma \ref{lemma: DG Approximation}, when the discriminator is identity function: $ D_v(t) = t $, that is
	\begin{align*}
		\left| {D}_v(G_u(\bh)) - {D}_v \left(\sum_{\bv} \gamma_{\bv} (\bh) G_u(\bv) \right) \right| &=
		\left\| G_u(\bh) - \sum_{\bv} \gamma_{\bv} (\bh) G_u(\bv) \right\|_2 \\
		&\leq L_{\bh} \| \bh -  \br(\bh) \|_2 + L_{G} \sum_{\bv \in \mC} |\gamma_{\bv} | \| \bv - \br(\bh) \|_2^2,
	\end{align*}
	then, we have
	\begin{align*}
		\left\| G_u \left(\sum_{\bv \in \mC} \gamma_{\bv}(\bh) \bv\right) - \sum_{\bv \in \mC} \gamma_{\bv} (\bh) G_u(\bv) \right\|_2
		&= \left\| G_u \left(\sum_{\bv \in \mC} \gamma_{\bv}(\bh) \bv\right) - G_u\left(\bh\right) + G_u\left(\bh\right) - \sum_{\bv \in \mC} \gamma_{\bv} (\bh) G_u(\bv) \right\|_2 \\
		&\leq \left\| G_u \left(\sum_{\bv \in \mC} \gamma_{\bv}(\bh) \bv\right) - G_u\left(\bh\right) \right\|_2 + \left\| G_u\left(\bh\right) - \sum_{\bv \in \mC} \gamma_{\bv} (\bh) G_u(\bv) \right\|_2 \\
		&\leq 2L_{\bh} \| \bh - \br(\bh) \|_2 + L_G \sum_{\bv \in \mC} |\gamma_{\bv} (\bh)| \| \bv - \br(\bh) \|_2^{2},
	\end{align*}
	where $ \br(\bh) = \sum_{\bv \in \mC} \gamma_{\bv}(\bh) \bv $.
\end{proof}

\section{Proof of Theorem \ref{theorem: Generalization Bound}}
In order to provide a generalization bound w.r.t. the neural net distance, we first give some relevant lemmas and theorems.
When the latent points lie on a latent manifold and the generator is Lipschitz smooth,  $ Q_{L_{\bh}, L_{G}} (\bgamma, \mC) $ has a bound as follows.
\begin{lemma}\textbf{\emph{(Manifold Coding \cite{yu2009nonlinear})}}
	\label{lemma: manifold_coding}
	If the latent points lie on a compact smooth manifold $ \mM $, given an $ (L_{\bh}, L_{G}) $-Lipschitz smooth generator $ G_u(\bh) $ and any $ \epsilon > 0 $, then there exist anchor points $ \mC \subset \mM $ and coding $ \bgamma $ such that 
	\begin{align*}
		Q_{L_{\bh}, L_{G}} (\bgamma, \mC)
		\leq \left[ L_{\bh}  c_{\mM} + \left(1 + \sqrt{d_{\mM}} + 4 \sqrt{d_{\mM}} \right) L_G \right] \epsilon^{2}.
	\end{align*}
\end{lemma}
Lemma \ref{lemma: manifold_coding} shows that the complexity of local coordinate coding depends on the intrinsic dimension of the manifold instead of the dimension of the basis.

\begin{*thm}{\textbf{\emph{\ref{theorem: Generalization Bound}} }}
	Suppose measuring function $ \phi(\cdot) $ is Lipschitz smooth: $ | \phi'(\cdot) | \leq L_{\phi} $, and bounded in $ [-\Delta, \Delta] $. Consider coordinate coding $ (\bgamma, \mC) $, an example set $ \mH $ in latent space and the empirical distribution $ \widehat{\mD}_{real} $, if the generator is Lipschitz smooth, then the expected generalization error satisfies the inequality:
	\begin{align*}
		\mmE_{\mH} \left[ d_{\mF, \phi} \left(\widehat{\mD}_{G_{\widehat{w}}\left( \bgamma(\bh) \right)}, \widehat{\mD}_{real} \right)
		\right] 
		\leq \inf_{ \mG } \mmE_{\mH} \left[ d_{\mF, \phi} \left( {\mD}_{ G_{u} (\bh)}, \widehat{\mD}_{real} \right) \right] + \epsilon(d_{\mM}),
	\end{align*}
	where $ \epsilon(d_{\mM}) = L_{\phi} Q_{L_{\bh}, L_{G}} (\bgamma, \mC) + 2\Delta $, and generative quality $ Q_{L_{\bh}, L_{G}} (\bgamma, \mC) $ is bounded w.r.t. $ d_{\mM} $ in Lemma \ref{lemma: manifold_coding} of supplementary material.
\end{*thm}

\begin{proof}
	Let $ \mH^{(k)} = \left\{ \bh_1^{(k)}, \bh_2^{(k)}, \ldots, \bh_r^{(k)} \right\} $ be a set of $ r $ latent samples which lie on the latent distribution. Consider $ n+1 $ independent experiments over the latent distribution, we have $ {\mH}_{r, n+1} = \left\{ \mH^{(1)}, \mH^{(2)}, \ldots, \mH^{(n+1)} \right\} $. Recall the optimization problem, we consider an empirical version of the expected loss:
	\begin{align}\label{optimization}
		[\widetilde{w}] = \argmin_{[w]} \left[ \frac{1}{n} \sum_{i=1}^{n+1} d_{\mF, \phi} \left({\mD}_{G_{w, \mH^{(i)}} (\bgamma(\bh))}, \widehat{\mD}_{real} \right) \right].
	\end{align}
	
	Let $ k $ be an integer randomly drawn from $ \{1, 2, \ldots, n+1\} $. Let $ \left[ \widehat{w}^{(k)} \right] $ be the solution of
	\begin{align}
		\left[\widehat{w}^{(k)}\right] = \argmin_{[w]} \left[ \frac{1}{n} \sum_{i\neq k}^{n+1} d_{\mF, \phi} \left({\mD}_{G_{w, \mH^{(i)}} (\bgamma(\bh))}, \widehat{\mD}_{real} \right) \right],
	\end{align}
	with the $ k $-th example left-out.
	
	Recall the definition of the neural net distance, we have
	\begin{align*}
		d_{\mF, \phi} (\mu, \nu) = \sup\limits_{\mF} \left| \mathop \mmE\limits_{\bx \sim \mu} \left[ \phi(D_v(\bx)) \right] + \mathop \mmE\limits_{\bx \sim \nu} \left[ \phi(\widetilde{D}_v (\bx)) \right] \right|,
	\end{align*}
	where $ \mF = \{ D_v, v \in \mV \} $.
	Given the $ k $-th sample experiment, the same real distribution $ \widehat{\mD}_{real} $ over the training samples $ \bx_1, \bx_2, \ldots, \bx_m $, and two different distributions generated by $ {G_{\widehat{w}^{(k)}, \mH^{(k)}}\left(\bgamma(\bh)\right)} $ and $ {G_{\widetilde{w}, \mH^{(k)}}\left(\bgamma(\bh)\right)} $, respectively, the difference value of the neural net distance between these two generated distributions is:
	\begin{align*}
		&d_{\mF, \phi} \left( \widehat{\mD}_{{G_{\widehat{w}^{(k)}, \mH^{(k)}}\left(\bgamma(\bh)\right)}}, \widehat{\mD}_{real} \right)
		- d_{\mF, \phi} \left(\widehat{\mD}_{{G_{\widetilde{w}, \mH^{(k)}}\left(\bgamma(\bh)\right)}}, \widehat{\mD}_{real} \right) \\
		=&\sup\limits \left| \mathop \mmE\nolimits_{\bx \in \widehat{\mD}_{real}} \left[ \phi(D_v(\bx)) \right] + \mathop \mmE\nolimits_{\bh \in \mH^{(k)}} \left[ \phi\left(\widetilde{D}_v \left( {G_{\widehat{w}^{(k)}, \mH^{(k)}}\left(\bgamma(\bh)\right)} \right)\right) \right] \right| \\
		&- \sup\limits \left| \mathop \mmE\nolimits_{\bx \in \widehat{\mD}_{real}} \left[ \phi(D_v(\bx)) \right] + \mathop \mmE\nolimits_{\bh \in \mH^{(k)}} \left[ \phi\left(\widetilde{D}_v \left( {G_{\widetilde{w}, \mH^{(k)}}\left(\bgamma(\bh)\right)} \right)\right) \right] \right|\\
		\leq& \sup \left| \mmE_{ \bh \in \mH^{(k)} } \left[ \phi \left( \widetilde{D}_{v} \left( {G_{\widehat{w}^{(k)}, \mH^{(k)}}\left(\bgamma(\bh)\right)} \right) \right) \right]
		- \mmE_{ \bh \in \mH^{(k)} } \left[ \phi \left( \widetilde{D}_{v} \left( {G_{\widetilde{w}, \mH^{(k)}}\left(\bgamma(\bh)\right)} \right) \right) \right] \right| \\
		=& \sup \left| \frac{1}{\left| \mH^{(k)} \right|} \sum\limits_{\bh \in \mH^{(k)}} \left[ \phi \left( \widetilde{D}_{v} \left( {G_{\widehat{w}^{(k)}, \mH^{(k)}}\left(\bgamma(\bh)\right)} \right) \right)
		- \phi \left( \widetilde{D}_{v} \left( {G_{\widetilde{w}, \mH^{(k)}}\left(\bgamma(\bh)\right)} \right) \right) \right] \right|
		\leq 2\Delta,
	\end{align*}
	
	where $ \widetilde{D}_v(\cdot) = 1 - D_v(\cdot) $. In the above derivation, the first equality uses the definition of the neural net distance. The last inequality holds by the assumption that $ \phi(\cdot) $ is $ L_{\phi} $-Lipschitz and bounded in $ [-\Delta, \Delta] $.
	
	By summing over $ k $, and consider any fixed $ G_u \in \mG $, we obtain:
	\begin{align*}
		\sum_{k=1}^{n+1} d_{\mF, \phi} \left( \widehat{\mD}_{ {G_{\widehat{w}^{(k)}, \mH^{(k)}}\left(\bgamma(\bh)\right)} }, \widehat{\mD}_{real} \right)	
		\leq& \sum_{k=1}^{n+1} d_{\mF, \phi} \left(\widehat{\mD}_{ {G_{\widetilde{w}, \mH^{(k)}}\left(\bgamma(\bh)\right)} }, \widehat{\mD}_{real} \right) + 2(n+1) \Delta \\
		\leq& \sum_{\bh \in \mH^{(k)}, k=1}^{n+1} d_{\mF, \phi} \left( \widehat{\mD}_{\sum_{\bv \in \mC} \gamma_{\bv} \left(\bh \right) G_u (\bv)}, \widehat{\mD}_{real} \right) + 2(n+1) \Delta \\
		\leq& \sum_{\bh \in \mH^{(k)}, k=1}^{n+1} d_{\mF, \phi} \left( \widehat{\mD}_{ G_{u} (\bh)}, \widehat{\mD}_{real} \right) + \sum_{k=1}^{n+1} L_{\phi} Q_{L_{\bh}, L_{G}} (\bgamma, \mC) + 2(n+1) \Delta,
	\end{align*}
	where $ {Q}_{L_{\bh}, L_{G}} (\bgamma, \mC) = \mmE_{\bh} \left[ L_{\bh} \| \bh -  \br(\bh) \|_2 + L_{G} \sum_{\bv \in \mC} |\gamma_{\bv} | \| \bv - \br(\bh) \|_2^2 \right] $. In the above derivation, the second
	inequality holds since $ \widetilde{w} $ is the minimizer of Problem (\ref{optimization}). The third inequality follows from the concavity of $ \phi(\cdot) $ and Lemma \ref{lemma: Generator Approximation}:
	
	\begin{align*}
		d_{\mF, \phi} \left( {\mD}_{\sum_{\bv \in \mC, \bh \in \mH^{(k)}} \gamma_{\bv} \left(\bh \right) G_u (\bv)}, \widehat{\mD}_{real} \right) =& \sup\limits \left| \mathop \mmE\nolimits_{\bx \in \widehat{\mD}_{real}} \left[ \phi(D_v(\bx)) \right] + \mathop \mmE\nolimits_{\bh \in \mH^{(k)}} \left[ \phi\left(\widetilde{D}_v \left(\sum\nolimits_{\bv \in \mC} \gamma_{\bv} \left(\bh \right) G_u (\bv) \right)\right) \right] \right| \\
		\leq& \sup\limits \left| \mathop \mmE\nolimits_{\bx \in \widehat{\mD}_{real}} \left[ \phi(D_v(\bx)) \right] + \mathop \mmE\nolimits_{\bh \in \mH^{(k)}} \left[ \phi\left(\widetilde{D}_v \left( G_u (\bh) \right) + \widehat{Q}_{L_{\bh}, L_{G}} (\bgamma, \mC) \right) \right] \right| \\
		\leq& \sup\limits \left| \mathop \mmE\nolimits_{\bx \in \widehat{\mD}_{real}} \left[ \phi(D_v(\bx)) \right] + \mathop \mmE\nolimits_{\bh \in \mH^{(k)}} \left[ \phi\left(\widetilde{D}_v \left( G_u (\bh) \right) \right) \right] \right| + L_{\phi} Q_{L_{\bh}, L_{G}} (\bgamma, \mC) \\
		=& d_{\mF, \phi} \left( {\mD}_{G_u (\bh)}, \widehat{\mD}_{real} \right) + L_{\phi} Q_{L_{\bh}, L_{G}} (\bgamma, \mC),
	\end{align*}
	where $ \widehat{Q}_{L_{\bh}, L_{G}} (\bgamma, \mC) = L_{\bh} \| \bh -  \br(\bh) \|_2 + L_{G} \sum_{\bv \in \mC} |\gamma_{\bv} | \| \bv - \br(\bh) \|_2^2 $ and $ \mmE_{\bh} \left[ \widehat{Q}_{L_{\bh}, L_{G}} (\bgamma, \mC) \right] = Q_{L_{\bh}, L_{G}} (\bgamma, \mC) $. In the above derivation, the firth equality holds by the definition of the neural net distance. The first inequality because of Lemma \ref{lemma: Generator Approximation} and the fact that $ \phi(\cdot) $ is a concave measuring function in Definition \ref{definition: F_distance}. Here, we suppose $ \phi(\cdot) $ is a monotonically increasing function. The second inequality holds by the following derivation:
	\begin{align*}
		&\left| \phi\left(\widetilde{D}_v \left( G_u (\bh) \right) + \widehat{Q}_{L_{\bh}, L_{G}} (\bgamma, \mC) \right) - \phi \left( \widetilde{D}_{v} \left( G_{u}(\bh) \right) \right) \right| \\
		\leq& \left| \phi' \left( \widetilde{D}_{v} \left( G_{u}(\bh) \right) \right) \left[ \left( \widetilde{D}_v \left( G_u (\bh) \right) + \widehat{Q}_{L_{\bh}, L_{G}} (\bgamma, \mC) \right) -  \widetilde{D}_{v} \left( G_{u}(\bh) \right) \right] \right| \\
		=& \left| \phi' \left( \widetilde{D}_{v} \left( G_{u}(\bh) \right) \right) \right| \widehat{Q}_{L_{\bh}, L_{G}} (\bgamma, \mC) \\
		\leq& L_{\phi} \widehat{Q}_{L_{\bh}, L_{G}} (\bgamma, \mC),
	\end{align*}
	In the above derivation, the first inequality uses the concavity of measuring function $ \phi(\cdot) $. The last inequality follows from that $ |\phi'| \leq L_{\phi} $. Now by taking expectation w.r.t. $ \mH_{r, n+1} $, we obtain
	\begin{align*}
		&\mmE_{\mH \subseteq \mH_{r, n+1}} \left[ d_{\mF, \phi} \left(\widehat{\mD}_{G_{\widehat{w}, \mH}\left( \bgamma(\bh) \right)}, \widehat{\mD}_{real} \right) \right] \\
		\leq& \mmE_{\mH \subseteq \mH_{r, n+1}} \left[ d_{\mF, \phi} \left( \widehat{\mD}_{ G_{u, \bh \in \mH} (\bh)}, \widehat{\mD}_{real} \right) \right] + L_{\phi} Q_{L_{\bh}, L_{G}} (\bgamma, \mC) + 2\Delta.
	\end{align*}
\end{proof}

\section{Proof of Theorem \ref{theorem: generalization_Rademacher}}
\begin{*thm} {\textbf{\emph{\ref{theorem: generalization_Rademacher}} }} 
	Under the condition of Theorem \ref{theorem: Generalization Bound}, and given an empirical distribution $ \widehat{\mD}_{real} $ drawn from $ \mD_{real} $, then the following holds with probability at least $ 1 - \delta $, 
	\begin{align*}
		\left| \mmE_{\mH} \left[d_{\mF, \phi} \left(\widehat{\mD}_{G_{\widehat{w}}}, {\mD}_{real} \right) \right]
		- \inf_{ \mG } \mmE_{\mH} \left[ d_{\mF, \phi} \left({\mD}_{G_u},  {\mD}_{real} \right) \right] \right| 
		\leq 2 {R}_{\mX}(\mF) + 2 \Delta \sqrt{\frac{2}{N} \log(\frac{1}{\delta})} + 2\epsilon(d_{\mM}),
	\end{align*}
	where $ {R}_{\mX}(\mF) = \mathop\mmE\limits_{\sigma, \mX} \left[ \sup\limits_{\mF} \frac{1}{N} \sum\limits_{i=1}^{N} \sigma_i \phi\left( D_v(\bx_i) \right) \right] $ and $ \sigma_i \in \{-1, 1\}, i = 1, 2, \ldots, m $ are independent uniform random variables.
\end{*thm}
\begin{proof}
	For the real distribution $ \mD_{real} $, we are interested in the generalization error in term of the following neural net distance:
	\begin{align}
		&\left| \mmE_{\mH} \left[d_{\mF, \phi} \left(\widehat{\mD}_{G_{\widehat{w}}}, {\mD}_{real} \right) \right]
		- \inf_{ \mG } \mmE_{\mH} \left[ d_{\mF, \phi} \left({\mD}_{G_u},  {\mD}_{real} \right) \right] \right|\nonumber \\
		\leq& \left| \mmE_{\mH} \left[ d_{\mF, \phi}  \left(\widehat{\mD}_{G_{\widehat{w}}}, {\mD}_{real} \right) \right]
		- \mmE_{\mH} \left[ \inf_{\mG} d_{\mF, \phi} \left({\mD}_{G_u}, {\mD}_{real} \right) \right] \right| \nonumber \\
		=& \left| \mmE_{\mH} \left[ d_{\mF, \phi} \left(\widehat{\mD}_{G_{\widehat{w}}},  {\mD}_{real} \right)
		- d_{\mF, \phi} \left( \widehat{\mD}_{G_{\widehat{w}}}, \widehat{\mD}_{real} \right)
		+ d_{\mF, \phi} \left( \widehat{\mD}_{G_{\widehat{w}}}, \widehat{\mD}_{real} \right)
		- \inf_{\mG} d_{\mF, \phi} \left({\mD}_{G_u},  {\mD}_{real} \right) \right] \right| \nonumber \\
		\leq& \left| \mmE_{\mH} \left[ d_{\mF, \phi} \left( \widehat{\mD}_{G_{\widehat{w}}}, {\mD}_{real} \right)
		- d_{\mF, \phi} \left(\widehat{\mD}_{G_{\widehat{w}}}, \widehat{\mD}_{real} \right)
		+ \inf_{\mG} d_{\mF, \phi} \left({\mD}_{G_u},  \widehat{\mD}_{real} \right)
		- \inf_{\mG} d_{\mF, \phi} \left({\mD}_{G_u},  {\mD}_{real} \right)
		+ \epsilon(d_{\mM}) \right] \right| \nonumber \\
		\leq& 2 \mmE_{\mH} \left[ \sup_{\mG} \left| d_{\mF, \phi} \left({\mD}_{G_u},  {\mD}_{real} \right) - d_{\mF, \phi} \left({\mD}_{G_u},  \widehat{\mD}_{real} \right) \right| + \epsilon(d_{\mM}) \right] \nonumber \\
		=& 2 \mmE_{\mH} \left[ \sup_{\mG} \left| \sup_{D_v \in \mF} \left| \mathop\mmE\limits_{\bx \in \mD_{real}} \left[ \phi \left( D_v(\bx) \right) \right] + \mathop\mmE\limits_{\bx \in \mD_{G_u}} \left[ \phi \left( \widetilde{D}_v(\bx) \right) \right] \right|
		- \sup_{D_v \in \mF} \left| \mathop\mmE\limits_{\bx \in \widehat{\mD}_{real}} \left[ \phi \left( D_v(\bx) \right) \right] + \mathop\mmE\limits_{\bx \in \mD_{G_u}} \left[ \phi \left( \widetilde{D}_v(\bx) \right) \right] \right| \right| + \epsilon(d_{\mM}) \right] \nonumber \\
		\leq& 2 \sup_{D_v \in \mF} \left| \mathop\mmE\limits_{\bx \in \mD_{real}} \left[ \phi \left( D_v(\bx) \right) \right]
		- \mathop\mmE\limits_{\bx \in \widehat{\mD}_{real}} \left[ \phi \left( D_v(\bx) \right) \right] \right| + 2\epsilon(d_{\mM}) \label{thm2: eq1}.
	\end{align}
	In the above derivation, the first inequality holds by by Jensen's inequality and the concavity of the infimum function.
	The second inequality holds by Theorem \ref{theorem: Generalization Bound}. The third inequality satisfies when we take supremum w.r.t. $ G_u \in \mG $. The last inequality uses the definition of the neural net distance and holds by triangle inequality. This reduces the problem to bounding the distance
	\[ d'_{\mF} \left( \mD_{real}, \widehat{\mD}_{real} \right) := \sup_{D_v \in \mF} \left| \mmE_{\bx \in \mD_{real}} \left[ \phi \left( D_v(\bx) \right) \right] - \mmE_{\bx \in \widehat{\mD}_{real}} \left[ \phi \left( D_v(\bx) \right) \right] \right|, \]
	between the true distribution and its empirical distribution. This can be achieved by the uniform concentration bounds developed in statistical learning theory, and thus the distance $ d'_{\mF} \left( \mD_{real}, \widehat{\mD}_{real} \right) $ can be achieved by the Rademacher complexity.
	Let $ \bx_1, \bx_2, \ldots, \bx_N \in \mX $ be a set of $ N $ independent random samples in data space.  We introduce a function
	\begin{align*}
		h\left( \bx_1, \bx_2, \ldots, \bx_N \right) = \sup_{D_v \in \mF} \left| \mmE_{\bx \in \mD_{real}} \left[ \phi \left( D_v(\bx) \right) \right] - \mmE_{\bx \in \widehat{\mD}_{real}} \left[ \phi \left( D_v(\bx) \right) \right] \right|.
	\end{align*}
	Since measuring function $ \phi $ is Lipschitz and bounded in $ [-\Delta, \Delta] $, changing $ \bx_i $ to another independent sample $ \bx'_i $ can change the function $ h $ by no more than $ \frac{4\Delta}{ N } $, that is,
	\begin{align*}
		h\left( \bx_1, \ldots, \bx_i \ldots, \bx_N \right) - h\left( \bx_1, \ldots, \bx'_i, \ldots, \bx_N \right) \leq \frac{4\Delta}{N},
	\end{align*}
	for all $ i \in [1, N] $ and any points $ \bx_1, \ldots, \bx_N, \bx'_i \in \mX $.
	McDiarmid's inequality implies that with probability at least $ 1 - \delta $, the following inequality holds:
	\begin{align}
		&\sup_{D_v \in \mF} \left| \mmE_{\bx \in \mD_{real}} \left[ \phi \left( D_v(\bx) \right) \right] - \mmE_{\bx \in \widehat{\mD}_{real}} \left[ \phi \left( D_v(\bx) \right) \right] \right| \nonumber \\
		\leq & \mmE \left[ \sup_{D_v \in \mF} \left| \mmE_{\bx \in \mD_{real}} \left[ \phi \left( D_v(\bx) \right) \right] - \mmE_{\bx \in \widehat{\mD}_{real}} \left[ \phi \left( D_v(\bx) \right) \right] \right| \right]
		+ 2 \Delta \sqrt{\frac{2\log\left( \frac{1}{\delta} \right)}{N}}. \label{thm2: eq2}
	\end{align}
	From the bound on Rademacher complexity, we have
	\begin{align}
		&\mmE \left[ \sup_{D_v \in \mF} \left| \mmE_{\bx \in \mD_{real}} \left[ \phi \left( D_v(\bx) \right) \right] - \mmE_{\bx \in \widehat{\mD}_{real}} \left[ \phi \left( D_v(\bx) \right) \right] \right| \right] \nonumber \\
		\leq& 2 \mmE_{\sigma, \mX} \left[ \sup_{D_v \in \mF} \frac{1}{N} \sum_{i=1}^{N} \sigma_i \phi\left( D_v(\bx_i) \right) \right] = 2 R_{\mX} (\mF). \label{thm2: eq3}
	\end{align}
	Combining the inequalities (\ref{thm2: eq1}), (\ref{thm2: eq2}) and (\ref{thm2: eq3}), we have
	\begin{align*}
		& \mmE_{\mH} \left[d_{\mF, \phi} \left(\widehat{\mD}_{G_{\widehat{w}}}, {\mD}_{real} \right) \right] - \inf_{G_u} \mmE_{\mH} \left[ d_{\mF, \phi} \left({\mD}_{G_u},  {\mD}_{real} \right) \right] \leq 2 {R}_{\mX}(\mF) + 2 \Delta \sqrt{\frac{2 \log(\frac{1}{\delta})}{N}} + 2\epsilon(d_{\mM}).
	\end{align*}
\end{proof}

\twocolumn

\end{document}